\documentclass[letterpaper, 10 pt, conference]{ieeeconf}  

\usepackage{pdfsync}
\usepackage[usenames, dvipsnames]{color}

\IEEEoverridecommandlockouts                              

\usepackage{times}
\usepackage[pdftex,backref=page,colorlinks=true,citecolor=blue]{hyperref}
\usepackage{amsmath,amssymb}
\usepackage{xspace}
\usepackage{amsfonts}
\usepackage{multirow}
\usepackage[ruled]{algorithm}
\usepackage{subfig}
\usepackage{graphicx}
  \graphicspath{{figs/}}
  \DeclareGraphicsExtensions{.pdf,.jpg,.png,.eps}

\usepackage{multicol}
\usepackage{wrapfig}
\usepackage{amsfonts}
\usepackage{xspace}
\usepackage{float}
\usepackage{tikz}
\usepackage{enumerate}
\usepackage{fancybox}

\usetikzlibrary{positioning}
\usetikzlibrary{shadows}
\usetikzlibrary{arrows}
\usetikzlibrary{shapes}

\newcommand{\tuple}[1]{\ensuremath{( #1 )}\xspace}

\newcommand*{\upperRomannumeral}[1]{\uppercase\expandafter{\romannumeral#1}}

\newcommand{\secref}[1]{Section~\ref{#1}}
\newcommand{\defref}[1]{Definition~\ref{#1}}
\newcommand{\lemref}[1]{Lemma~\ref{#1}}
\newcommand{\thmref}[1]{Theorem~\ref{#1}}
\newcommand{\equref}[1]{Equation~\ref{#1}}
\newcommand{\algmref}[1]{Algorithm~\ref{#1}}
\newcommand{\cref}[1]{{\cite{#1}}}
\newcommand{\figref}[1]{Fig.~\ref{#1}}

\newcommand{\tabref}[1]{Table~\ref{#1}}

\newcommand{\asmref}[1]{Assumption~\ref{#1}}

\newcommand{\ie}{\textit{i.e.}}

\newcommand{\eg}{\textit{e.g.}}

\newcounter{cmt}

\newcommand{\rocksample}{Rocksample\xspace}
\newcommand{\battleship}{Battleship\xspace}

\newcommand{\his}{\ensuremath{h}\xspace}

\newcommand{\rewFp}{\ensuremath{R}\xspace}
\newcommand{\discount}{\ensuremath{\gamma}\xspace}

\newcommand{\actsph}{\ensuremath{A}\xspace}
\newcommand{\obssph}{\ensuremath{O}\xspace}
\newcommand{\stsph}{\ensuremath{X}\xspace}
\newcommand{\parsph}{\ensuremath{\Theta}\xspace}
\newcommand{\traFh}{\ensuremath{T}\xspace}
\newcommand{\obsFh}{\ensuremath{Z}\xspace}
\newcommand{\rewFh}{\ensuremath{R}\xspace}
\newcommand{\acth}{\ensuremath{a}\xspace}
\newcommand{\obsh}{\ensuremath{o}\xspace}
\newcommand{\sth}{\ensuremath{x}\xspace}
\newcommand{\parh}{\ensuremath{\theta}\xspace}

\newcommand{\actspm}{\ensuremath{A}\xspace}
\newcommand{\stspm}{\ensuremath{S}\xspace}
\newcommand{\obsspm}{\ensuremath{\tilde{O}\xspace}}
\newcommand{\obsm}{\ensuremath{\tilde{o}\xspace}}
\newcommand{\traFm}{\ensuremath{T}\xspace}
\newcommand{\rewFm}{\ensuremath{R}\xspace}
\newcommand{\actm}{\ensuremath{a}\xspace}
\newcommand{\stm}{\ensuremath{s}\xspace}
\newcommand{\parm}{\ensuremath{\theta}\xspace}

\newcommand{\nul}{\ensuremath{null}\xspace}
\newcommand{\plyh}{\ensuremath{\phi}\xspace}

\newcommand{\stspim}{\ensuremath{S}\xspace}
\newcommand{\actspim}{\ensuremath{A}\xspace}
\newcommand{\traFim}{\ensuremath{T}\xspace}
\newcommand{\rewFim}{\ensuremath{\tilde{R}}\xspace}
\newcommand{\actim}{\ensuremath{a}\xspace}
\newcommand{\stim}{\ensuremath{s}\xspace}

\newcommand{\mpr}{\ensuremath{\theta}\xspace}

\newcommand{\disRB}{\ensuremath{\beta}\xspace}
\newcommand{\disvfn}{\ensuremath{\gamma}\xspace}

\newcommand{\vrep}{V-REP\xspace}
\newcommand{\mmdp}{Mean MDP\xspace}

\newcommand{\uct}{UCT\xspace}

\newcommand{\pomcp}{POMCP\xspace}
\newcommand{\despot}{DESPOT\xspace}
\newcommand{\qmdp}{QMDP\xspace}
\newcommand{\sarsop}{SARSOP\xspace}
\newcommand{\pomdp}{POMDP\xspace}
\newcommand{\hpmdp}{POMDP-lite\xspace}
\newcommand{\hgmdp}{HGMDP\xspace}
\newcommand{\hipmdp}{HiP-MDP\xspace}
\newcommand{\brl}{BRL\xspace}
\newcommand{\pacmdp}{PAC-MDP\xspace}
\newcommand{\mdp}{MDP\xspace}
\newcommand{\online}{\emph{online}\xspace}
\newcommand{\offline}{\emph{offline}\xspace}

\newcommand{\knownMK}{\textbf{known state-action MDP}}
\newcommand{\Lone}{\ensuremath{L_1}\xspace}

\newcommand{\hpmdpM}{\ensuremath{\mathcal{H}}}
\newcommand{\collectionM}{\ensuremath{\mathcal{C}}}

\newcommand{\rebonus}{\ensuremath{RB}\xspace}
\newcommand{\vfn}{\ensuremath{V}\xspace}
\newcommand{\optvfn}{\ensuremath{V^*}\xspace}

\newcommand{\optvfni}{\ensuremath{\tilde{V}^*}\xspace}

\newcommand{\pivfn}{\ensuremath{V^{\pi}}\xspace}

\newcommand{\vfni}{\ensuremath{\tilde{V}}\xspace}
\newcommand{\qfni}{\ensuremath{\tilde{Q}}\xspace}
\newcommand{\escapeE}{\ensuremath{A_K}\xspace}

\newcommand{\bigO}{\ensuremath{O}\xspace}
\newcommand{\bigOhead}{\ensuremath{\tilde{O}}\xspace}

\newcommand{\sampCplx}{\ensuremath{\zeta}\xspace}
\newcommand{\sampAccu}{\ensuremath{\kappa}\xspace}

\newcommand{\plcypi}{\ensuremath{\pi}\xspace}
\newcommand{\optplcypi}{\ensuremath{\pi^*}\xspace}

\newcommand{\optply}{\ensuremath{\mathcal{A}}\xspace}
\DeclareMathOperator*{\argmax}{arg\,max}
\newcommand{\ev}{\ensuremath{\mathrm{E}}\xspace}

\newcommand{\mRsp}{\ensuremath{\mathbb{R}}\xspace}
\newcommand{\mTsp}{\ensuremath{\mathbb{T}}\xspace}
\newcommand{\accRB}{\ensuremath{\mathcal{RB}}\xspace}
\newcommand{\accRE}{\ensuremath{\mathcal{RE}}\xspace}
\newcommand{\vE}{\ensuremath{E}\xspace}
\newcommand{\vB}{\ensuremath{B}\xspace}
\newcommand{\upbvi}{\ensuremath{\frac{1}{1 - \gamma}}\xspace}

\newcommand{\ith}{\ensuremath{i^{th}}\xspace}

\usepackage[compatible]{algpseudocode}
\renewcommand{\algorithmiccomment}[1]{\bgroup\hfill\tiny//~#1\egroup}

\newtheorem{defn}{Definition}
\newtheorem{them_1}{Theorem}
\newtheorem{lemma}{Lemma}
\newtheorem{assum}{Assumption}

\title{\bf \hpmdp for Robust Robot Planning under Uncertainty}

\author{Min Chen \quad Emilio Frazzoli \quad David Hsu \quad Wee Sun Lee%
\thanks{Chen, Hsu, and Lee are with the Department of Computer
  Science, National University of Singapore, Singapore 117417, Singapore.
  Frazzoli is with the  Laboratory for Information \& Decision Systems, Massachusetts Institute of Technology, Cambridge, MA 02139, USA
}%
\thanks{Chen is supported by Singapore-MIT Alliance for Research \& Technology (SMART) graduate fellowship,
  Frazzoli is supported by the Singapore NRF through the SMART Future Urban Mobility IRG,
  Hsu is supported by
  A*STAR Industrial Robotics Program grant R-252-506-001-305
  and Lee is supported by AOARD grant FA2386-12-1-4031. The
  views and conclusions contained herein are those of the authors and should
  not be interpreted as necessarily representing the official policies or
  endorsements, either expressed or implied, of the Air Force Research
  Laboratory or the U.S. Government.}
}

\begin{document}

\maketitle
\thispagestyle{empty}
\pagestyle{empty}

\begin{abstract}
The partially observable Markov decision process (POMDP) provides a principled
general model for planning under uncertainty. However, solving a general POMDP
 is computationally intractable in the worst case. This paper introduces
\emph{\hpmdp},
a subclass of POMDPs in which the hidden state variables are constant
or only change deterministically.
We show that a \hpmdp is equivalent to a
set of fully observable Markov decision processes indexed by a hidden
parameter and is 
useful for modeling a variety of interesting robotic tasks.
We develop a simple model-based Bayesian reinforcement learning
algorithm to solve \hpmdp models.
The algorithm  performs well on large-scale \hpmdp
models with up to $10^{20}$ states and outperforms the state-of-the-art
general-purpose POMDP algorithms. We
further show that the algorithm is near-Bayesian-optimal
under suitable conditions.
\end{abstract}

\section{INTRODUCTION}
\label{sec:intro}

Imperfect robot control, sensor noise, and unexpected
environment changes all contribute to uncertainties and pose
significant challenges to robust robot planning.
Robots must \emph{explore} in order  to gain
information and reduce uncertainty. At the same time, they must
\emph{exploit} the information to achieve
task objectives.  The \emph{partially observable Markov decision process}
(POMDP) \cite{KaeLit98,Son71} provides a principled general framework to
balance exploration and exploitation optimally. It has found application in
many robotic tasks, ranging from navigation~\cite{RoyThr99},
manipulation~\cite{HsiKae07,KovPol15} to human-robot
interaction~\cite{NikRam15}.  However, solving POMDPs exactly is
computationally intractable in the worst
case~\cite{papadimitriou1987complexity}.  While there has been rapid progress
on efficient approximate POMDP algorithms in recent years (\eg,
\cite{SmiSim05,KurHsu08,SilVen10,SomYe13,SeiKur15}), it remains a challenge to
scale  up to very large
POMDPs with complex dynamics.  


The complexity of a POMDP lies in the system dynamics, partial observability,
and particularly, the confluence of the two.  
We introduce \emph{\hpmdp}, a
factored model that restricts partial observability to state variables
that are constant or change deterministically. 
While this may appear restrictive, the \hpmdp is 
powerful enough to model a variety of interesting robotic tasks:
\begin{itemize}
\item \emph{Unknown system parameters}.
  A double-pendulum acrobot with link mass unknown \textit{a priori}
  swings up to the standup
  configuration~\cite{BaiHsu13a,DosKon13}.    
\item \emph{Unknown types}.
  In autonomous driving,
  a robot vehicle encounters
  a human driver of unknown behavior
  at an uncontrolled traffic intersection~\cite{BanWon12}.
  In a human-robot collaborative task, the robot poses an object to the human
  according to
  unknown human preference~\cite{NikRam15}. 
\item \emph{Unknown goals}. An assistive agent helps a human cooking one of
  several dishes, without knowing the human's intention in advance~\cite{FerNat07}.
\end{itemize}
These tasks all require the robot to gather information on the unknown
quantities from noisy observations,
while achieving the task objective at the same time. 
They in fact belong to a special case, in which the hidden variables
remain \emph{constant} throughout.
We mainly focus on this special case here.

Interestingly, the famous \emph{Tiger} problem, which appeared in a seminal
paper on POMDPs~\cref{KaeLit98}, also belongs to this special case, after a
small modification. In Tiger, an agent stands in front
of two closed doors. A tiger is behind one of the doors. The agent's objective
is to open the door without the tiger.  In the POMDP model, the state is the
unknown tiger position. The agent has three actions: open the left door (OL),
open the right door (OR), and listen (LS). OL and OR produce no
observation. LS produces a noisy observation, tiger left (TL) or tiger right
(TR), each correct with probability $0.85$. Listening has a cost of $-1$.  If
the agent opens the door with no tiger, it gets a reward of $10$; otherwise,
it incurs a penalty of $-100$.  To perform well, the agent must decide on the
optimal number of listening actions before taking the open action.
While Tiger is a toy problem,  it captures the essence of robust
planning under uncertainty: trade off gathering information and exploiting the
information to achieve the task objective.  The original
Tiger is a repeated game. Once the agent opens a door, the game resets with 
the tiger going behind the two doors with
equal probability. We change it into a one-shot game: the game terminates
once the agent opens a door. The one-shot game has a single state variable,
the tiger position, which remains unchanged during the game, and thus admits a
\hpmdp model. The repeated game is a POMDP, but not a \hpmdp.

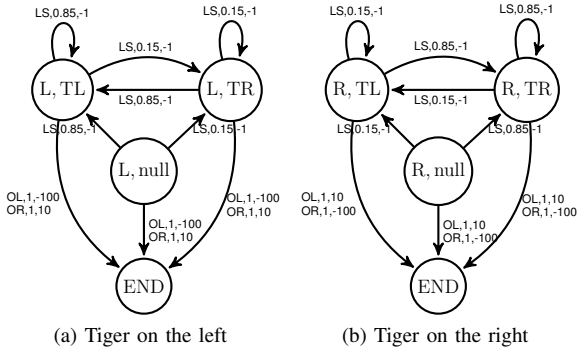
\begin{figure}
\centering
\begin{tabular}{@{\hspace{0em}} c @{\hspace{0em}} c @{\hspace{0em}}}
\subfloat[Tiger on the left]{
\begin{tikzpicture}[->,>=stealth',shorten >=1pt,auto,node distance=3cm,
thick,main node/.style={scale=0.51,circle,fill=white!20,draw,font=\sffamily\Large\bfseries}]
s
  \node[main node] (2) {$\mathrm{L},\mathrm{TL}$};
  \node[main node] (1) [below right of=2] {$\mathrm{L}, \mathrm{null}$};
  \node[main node] (3) [right=1.4cm of 2] {$\mathrm{L},\mathrm{TR}$};
  \node[main node] (4) [below of = 1] {$\mathrm{END}$};

  \path[every node/.style={scale=0.5,font=\sffamily\small}]
    (1) edge node [left] {LS,0.85,-1} (2)
      edge node [right] {LS,0.15,-1} (3)
      edge node [text width=1.5cm, align=left] {OL,1,-100\\OR,1,10} (4)
    (2) edge [loop above] node {LS,0.85,-1} (2)
      edge [bend left] node {LS,0.15,-1} (3)
      edge [bend right] node [text width=1.5cm, left] {OL,1,-100\\OR,1,10} (4)
    (3) edge [loop above] node {LS,0.15,-1} (3)
      edge node {LS,0.85,-1} (2)
      edge [bend left] node [text width=1.5cm, right] {OL,1,-100\\OR,1,10} (4);
\end{tikzpicture}} & 
\subfloat[Tiger on the right]{
\begin{tikzpicture}[->,>=stealth',shorten >=1pt,auto,node distance=3cm,
thick,main node/.style={scale=0.51,circle,fill=white!20,draw,font=\sffamily\Large\bfseries}]

  \node[main node] (2) {$\mathrm{R},\mathrm{TL}$};
  \node[main node] (1) [below right of=2] {$\mathrm{R}, \mathrm{null}$};
  \node[main node] (3) [right=1.4cm of 2] {$\mathrm{R},\mathrm{TR}$};
  \node[main node] (4) [below of = 1] {$\mathrm{END}$};

  \path[every node/.style={scale=0.5,font=\sffamily\small}]
    (1) edge node [left] {LS,0.15,-1} (2)
      edge node [right] {LS,0.85,-1} (3)
      edge node [text width=1.5cm, align=left] {OL,1,10\\OR,1,-100} (4)
    (2) edge [loop above] node {LS,0.15,-1} (2)
      edge [bend left] node {LS,0.85,-1} (3)
      edge [bend right] node [text width=1.5cm, left] {OL,1,10\\OR,1,-100} (4)
    (3) edge [loop above] node {LS,0.85,-1} (3)
      edge node {LS,0.15,-1} (2)
      edge [bend left] node [text width=1.5cm, right] {OL,1,10\\OR,1,-100} (4);
\end{tikzpicture}}
\end{tabular}
\caption{A POMDP model for one-shot Tiger is transformed into a set of
  two MDPs. A node is labeled by a pair, representing a POMDP state and an
  observation. The start state has null observation and is labeled
  accordingly.
  The special  terminal state is labeled as ``END''.
  An edge is labeled by a triple, representing the action , the probability of
  reaching the next state with the action, and the reward. }
\label{fig:TigerMdl}
\end{figure}

A \hpmdp is \emph{equivalent} to a
set of \emph{Markov decision processes} (MDPs) indexed by a hidden parameter.
The key idea for the equivalence transformation is to 
combine a POMDP state and an observation to form an expanded MDP state, and
capture both POMDP state-transition uncertainty and observation
uncertainty in the MDP transition dynamics. In the one-shot
Tiger example, we form two MDPs indexed by the tiger
position, left (L) or right (R) (\figref{fig:TigerMdl}). An MDP state is a
pair, consisting of a POMDP state and an observation. For example, in the MDP
with the tiger on the left, we have $(\text{L}, \text{TL})$, which represents that the true tiger
position is L and the agent  receives the observation TL. If the agent
takes the action LS, with probability $0.15$, we transit to
the new state $(\text{L}, \text{TR})$ and receives observation TR.
See \secref{sec:hpmdp} for details of  the general construction.

The equivalence enables us to develop an online algorithm for
POMDP-lite through model-based Bayesian reinforcement learning (RL). 
If the hidden
parameter value were known, our problem would simply become  an MDP, which has
well-established algorithms. To gather information on the unknown hidden parameter,
the robot must explore.  It maintains a \emph{belief}, \ie,
a probability distribution over the hidden
parameter and follows the internal reward approach for model-based Bayesian
RL~\cite{kolter2009near,sorg2012variance},
which modifies the
MDP reward function in order to encourage exploration.
At each time step, the online algorithm solves an
internal reward MDP to choose an action and then updates the belief
to incorporate the new observation received.
Our algorithm is simple to
implement. It performs well on large-scale \hpmdp
tasks with up to $10^{20}$ states and outperforms the state-of-the-art
general-purpose POMDP algorithms. 
Furthermore, it is 
near-Bayesian-optimal under suitable conditions.

\section{Related Work}
\label{sec:related}
POMDP planning has a huge literature (see, \eg, \cite{KaeLit98,KurHsu08,Son71,SmiSim05,KurHsu08,SilVen10,SomYe13,SeiKur15} ). 
Our brief review focuses on  online search algorithms.  At each time step, an
online algorithm performs a look-ahead search and computes a best action for
the current belief only~\cite{RosPin08}.  After the robot executes the action,
the algorithm updates the belief based on the observation received.  The
process then repeats at the new belief for the next time step.
Online search algorithms scale up by focusing on the current belief only,
rather than all possible beliefs that the robot may encounter. Further, since
online algorithms recompute a best action from scratch at each step, they
naturally handle unexpected environment changes without additional overhead.
POMCP~\cite{SilVen10} and DESPOT~\cite{SomYe13} are the fastest online POMDP
algorithms available today.  Both employ the idea of \emph{sampling} future
contingencies.  POMCP performs Monte Carlo tree search (MCTS). It has low
overhead and scales up to very large POMDPs, but it has extremely poor
worst-case performance, because MCTS is sometimes overly greedy. DESPOT
samples a fixed number of future contingencies deterministically in advance and
performs heuristic search on the resulting search tree. This substantially
improves the worst-case performance bound. It is also more flexible and easily
incorporates domain knowledge. DESPOT has been successfully implemented for
real-time autonomous driving  in a crowd~\cite{BaiCai15}. It is  also a crucial
component in a system that won the \emph{Humanitarian Robotics and Automation
  Technology Challenge} (HRATC) 2015 on a demining task.

Instead of solving the general POMDP, we take a different approach and
identify a structural property that enables simpler and more efficient
algorithms through model-based Bayesian RL.  Like POMDP-lite, the \emph{mixed
  observability Markov decision process} (MOMDP)~\cite{ong2010planning} is
also a factored model. However, it does not place any restriction on partially
observable state variables. It is in fact equivalent to the general POMDP, as
every POMDP can be represented as a MOMDP and \textit{vice versa}.  The
\emph{hidden goal Markov decision process} (\hgmdp)~\cite{FerNat07} and the
\emph{hidden parameter Markov decision process} (\hipmdp)~\cite{DosKon13} are related
to \hpmdp. They both restrict partially observability to static hidden
variables. The work on \hgmdp relies on a myopic heuristic for planning, 
and it is unlikely to perform well on tasks that need exploration.
The work on \hipmdp focuses mainly on learning the hidden structure from data.

There are several approaches to  Bayesian RL~\cite{asmuth2009bayesian,poupart2006analytic,strehl2012incremental,kolter2009near}. The internal reward
approach is among the most successful. It is simple and performs well in
practice. 
Internal reward methods can be further divided into two main categories,
PAC-MDP and Bayesian optimal. PAC-MDP algorithms are optimal with respect to
the true MDP~\cref{kearns2002near, strehl2012incremental, sorg2012variance}.
They provide strong theoretical guarantee, but may suffer from over
exploration empirically.  Bayesian optimal algorithms are optimal with respect
to the optimal Bayesian policy.  They simply try to achieve high expected
total reward. In particular, the Bayesian Exploration Bonus
(BEB)~\cite{kolter2009near} algorithm achieves lower sample complexity than
the PAC-MDP algorithms.  However, BEB requires a Dirichlet prior on the hidden
parameters.  Our algorithm is inspired by BEB, but constructs the exploration
bonus differently. It allows arbitrary discrete prior, a very useful feature in
practice.

\section{\hpmdp}
\label{sec:hpmdp}

\subsection{Definition}
\hpmdp is a special class of \pomdp with a ``deterministic assumption'' on its partially observable variable, specifically, the partially observable variable in \hpmdp is static or has deterministic dynamic. Formally we introduce \hpmdp as a tuple \tuple{\stsph, \parsph, \actsph, \obssph, \traFh, \obsFh, \rewFh, \discount}, where \stsph is a set of fully observable states, \parsph is the hidden parameter which has finite number of possible values: $\parsph = \{\parh_1,\parh_2,...,\parh_N\}$, the state space is a cross product of fully observable states and hidden parameter: $\stsph \times \parsph = \{ \tuple{\sth_i, \parh_j} | \sth_i \in \stsph, \parh_j \in \parsph \}$. \actsph is a set of actions, \obssph is a set of observations. The transition function $\traFh(\sth, \parh, \acth, \sth', \parh') = P(\sth' | \parh, \sth, \acth) P(\parh' | \parh, \sth, \acth)$ for $\sth, \sth' \in \stsph, \parh, \parh' \in \parsph, \acth \in \actsph$ specifies the probability of reaching state $\tuple{\sth', \parh'}$ when the agent takes action \acth at state \tuple{\sth, \parh}, where $\parh' = \parh$ or $P(\parh' | \parh, \sth, \acth) = 1$ according to the ``deterministic assumption''. The observation function $\obsFh(\sth', \parh', \acth, \obsh) = P(\obsh | \parh', \sth', \acth)$ specifies  the probability of receiving observation \obsh after taking action \acth and reaching state $\tuple{\sth', \parh'}$. The reward function $\rewFh(\sth, \parh, \acth)$ specifies the reward received when the agent takes action \acth at state \tuple{\sth, \parh}. \discount is the discount factor. 

In \hpmdp, the state is unknown and the agent maintains a belief $b$, which is a probability distribution over the states. In each step, the agent takes an action $\acth$ and receives a new observation, the belief is updated according to Bayes' rule, $b' = \tau(b, a, \obsh)$. The solution to a \hpmdp is a policy \plcypi which maps belief states to actions, \ie, $\plcypi(b) = a$. The value of a policy \plcypi is the expected reward with respect to the initial belief $b_0$: $\pivfn(b_0) = \ev\big{[}\underset{t = 0}{\overset{\infty}{\sum}} \discount^t \rewFp(\sth_t, \parh_t, \acth_t) \big{]}$, where $\tuple{\sth_t, \parh_t}$ and $\acth_t$ denote the state and action at time $t$. An optimal policy has the highest value in all belief states, \ie, $\optvfn(b) \geq \pivfn(b), \forall \plcypi, b$, and the corresponding optimal value function satisfies Bellman's equation:
\vspace{-0.8em}
\begin{align*}
\optvfn(b) &= \underset{\acth}{\max}\bigg{\{} \underset{\sth \in \stsph, \parh \in \parsph}{\sum}b(\sth, \parh)\rewFh(\sth, \parh, \acth) \\ 
  &\qquad \qquad \qquad+ \discount \underset{\obsh \in \obssph}{\sum}P(\obsh | b, \acth) \optvfn \big{(}\tau(b, a, \obsh)\big{)} \bigg{\}}
\end{align*}

\subsection{Equivalent Transformation to a Set of MDPs}
In this section, we show an important property of \hpmdp model that it is equivalent to a collection of MDPs indexed by $\parm$. A \mdp model with parameter $\parm_i$ is a tuple $\tuple{\stspm, \actspm, \traFm, \rewFm, \discount, \parm_i}$, where $\stspm$ is a set of states, $\actspm$ is a set of actions, $\traFm(\parm_i, \stm, \actm, \stm')$ is the transition function, $\rewFm(\parm_i, \stm, \actm)$ is the reward function, $\discount$ is the discount factor. 

\begin{them_1}
\label{thm:hpmdp}
Let $\hpmdpM = \tuple{\stsph, \parsph, \actsph, \obssph, \traFh, \obsFh, \rewFh, \discount}$ be a \hpmdp model, where $\parsph = \{\parh_1,\parh_2,...,\parh_N\}$. It equals to a collection of MDPs indexed by $\parm$, $\collectionM = \underset{\parm_i \in \parsph}{\bigcup} \tuple{\stspm, \actspm, \traFm, \rewFm, \discount, \parm_i}$.
\end{them_1}

\tikzset{
    state/.style = {draw, font=\sffamily, circle, align=center, minimum size=8mm},
    hidden_parameter/.style = {draw, font=\sffamily, circle, align=center, minimum size=8mm},
    action/.style = {draw, font=\sffamily, diamond, align=center, minimum size=8mm},
    observation/.style = {draw, font=\sffamily, rectangle, align=center, minimum size=7mm},
    reward/.style = {draw, font=\sffamily, ellipse, align=center, minimum height=6mm, minimum width=9mm}
}
\begin{figure}
\centering
  \begin{tikzpicture}[scale = 0.73]
  \node at (-2.2,-1.5) {\hpmdp};
  \draw [dashed] (-4.6,-1.1) rectangle (-3.4,1.1);
  \node[state] (s_pomdp) at (-4,0.55) {$\sth$};
  \node[hidden_parameter] (hi_pomdp) at (-4,-0.55) {$\parh$};
  \draw [dashed] (-1.1,-1.1) rectangle (0.1,1.1);
  \node[state] (s_pomdp_1) at (-0.5,0.55) {$\sth'$};
  \node[hidden_parameter] (hi_pomdp_1) at (-0.5,-0.55) {$\parh'$};
  \node[action] (a_pomdp) at (-2.5, 2.5) {$a$};
  \node[observation] (obs_pomdp) at (-0.5,2.5) {$o$};
  \node[reward] (r_pomdp) at (-5,2.5) {$R$};
  \draw [->, draw=blue] (-3.4,0) to (-1.1,0);
  \draw [->, draw=blue] (a_pomdp) to[out=270,in=180] (-1.1,0.85);
  \draw [->, draw=blue] (-0.5,1.1) to (obs_pomdp);
  \draw [->, draw=blue] (a_pomdp) to (obs_pomdp);
  \draw [->, draw=blue] (a_pomdp) to (r_pomdp);
  \draw [->, draw=blue] (-4,1.1) to[in=270,out=90] (r_pomdp);

  \node at (3.7,-1.5) {$MDP_{\parh_i}$};
  \draw [dashed] (1.4,-1.1) rectangle (2.6,1.1);
  \node[state] (s_hpmdp) at (2,0.55) {$\sth$};
  \node[observation] (obs_hpmdp) at (2,-0.55) {$\obsm$};
  \draw [dashed] (4.9,-1.1) rectangle (6.1, 1.1);
  \node[state] (s_hpmdp_1) at (5.5,0.55) {$\sth'$};
  \node[observation] (obs_hpmdp_1) at (5.5,-0.55) {$\obsm'$};
  \node[action] (a_hpmdp) at (3.5, 2.5) {$a$};
  \node[hidden_parameter] (hi_hpmdp) at (5.5,2.5) {$\parh_i$};
  \node[reward] (r_hpmdp) at (1,2.5) {$R$};
  \draw [->, draw=blue] (2.6,0)  to  (4.9,0) ;
  \draw [->, draw=blue] (a_hpmdp) to[out=270,in=180] (4.9,0.8);
  \draw [->, draw=blue] (hi_hpmdp) to (5.5,1.1);
  \draw [->, draw=blue] (a_hpmdp) to (r_hpmdp);
  \draw [->, draw=blue] (2,1.1) to[in=270,out=90] (r_hpmdp);
  \draw [->, draw=blue] (hi_hpmdp) to[in=30,out=150] (r_hpmdp);
\end{tikzpicture}
\caption{Graphic model for \hpmdp (left) and the MDP model with parameter $\parm_i$ (right).}
\vspace{-0.5em}
\label{fig:graphic_model}
\end{figure}
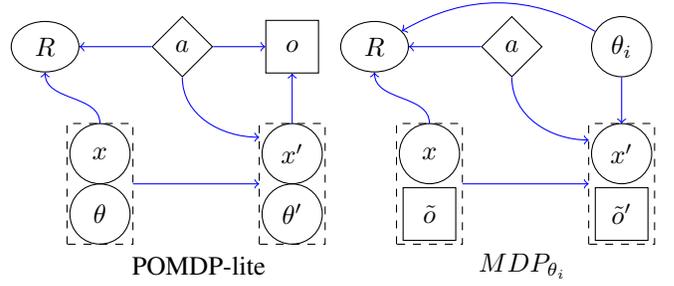

\begin{proof}[Proof of~\thmref{thm:hpmdp}]
To show the equivalence between $\hpmdpM$ and $\collectionM$, we first reduce $\collectionM$ to $\hpmdpM$. This direction is easy, we can simply treat $\parm_i$ as part of the state in a \hpmdp model, the remaining part can become part of a \hpmdp model without change. The more interesting direction is to reduce $\hpmdpM$ to $\collectionM$.

Let's first consider the case when the value of $\parm$ remains constant. Given $\parh = \parh_i$, the \hpmdp model becomes a \mdp model with parameter $\parm_i$, which consists of the following elements: \tuple{\stspm, \actspm, \traFm, \rewFm, \discount, \parm_i}. Where \stspm is the state space: $\stspm = \stsph \times \obsspm$, $\obsspm = \obssph \cup \{\nul\}$, in which \nul simply means no observation received. \actspm is the set of actions which is identical to the actions in the \hpmdp model. The transition function $\traFm(\parh_i, \stm, \actm, \stm') = P(\sth', \obsm' | \parh_i, \sth, \obsm, \actm) = P(\obsm' | \parh_i, \sth', \actm) P(\sth' | \parh_i, \sth, \actm)$ specifies the probability of reaching state $\stm'$ after taking action $\actm$ at state \stm, where $P(\sth' | \parh_i, \sth, \actm)$ and $P(\obsm' | \parh_i, \sth', \actm)$ are the transition and observation probability function in the \hpmdp model. The reward function $\rewFm(\parm_i, \stm, \actm) = \rewFh(\parh_i, \sth, \actm)$ specifies the reward received when the agent takes action \actm at state \stm. \discount is the discount factor. The graphic model in~\figref{fig:graphic_model} shows the relationship between \hpmdp model and the corresponding \mdp model with parameter $\parm_i$. Since the hidden parameter has finite number of values, a \hpmdp can be reduced to a collection of MDPs indexed by $\parm$. 

Then, we show that a simple extension allows us to handle the case when the value of the hidden variable changes deterministically. The key intuition here is that the deterministic dynamic of hidden variable does not introduce additional uncertainties into the model, \ie, given the initial value of hidden variable $\parm^0$, and the history up to any time step $t$, $\his_t = \{\sth_0, \acth_0, \sth_1, \acth_1, ... , \sth_{t-1}, \acth_{t-1}\}$, the value of hidden variable can be predicted using a deterministic function, $\parm^{t} = \plyh(\parm^0, \his_t)$. Thus given the initial value of hidden variable $\parm^0 = \parm^0_i$ and a deterministic function $\plyh$, a \hpmdp model can be reduced to a \mdp model: \tuple{\stspm, \actspm, \traFm, \rewFm, \discount, \parm^0_i}. Compared with the static case, the state here is further augmented by the history $\his_t$ \ie, $\stm_t = \tuple{\sth_t, \obsm_t, \his_t}$. The value of $\parm^t$ was fully captured by $\stm_t$, $\parm^0_i$ and $\plyh$, $\parm^t = \plyh(\parm^0_i, \stm_t)$. The rest of the \mdp model is similar to the static case. In particular, the set of actions $\actspm$ is identical to the \hpmdp model, the transition function is $\traFm(\parm^0_i, \stm_t, \actm_t, \stm_{t+1}) = P(\obsm_{t+1} | \parm^{t+1}, \sth_{t+1}, \actm_t) P(\sth_{t+1} | \parm^t, \sth_t, \acth_t)$, the reward function is $\rewFm(\parm^0_i, \stm_t, \actm_t) = \rewFh(\parh^t, \sth_t, \actm_t)$, \discount is the discount factor. Since $\parm^0$ has finite number of values, a \hpmdp can be reduced to a collection of MDPs indexed by $\parm^0$.
\end{proof}

\subsection{Algorithm}
In this part, we present an efficient model based \brl algorithm for \hpmdp. The solution to the \brl problem is a policy \plcypi, which maps a tuple \tuple{belief, state} to actions, \ie, $\actm = \plcypi(b, s)$. The value of a policy $\plcypi$ for a belief $b$ and state $\stm$ is given by Bellman's equation
\begin{align*}
\pivfn(b, s) &= R(b, s, a) + \discount \underset{b', s'}{\sum}P(b', s' | b, s, a)\pivfn(b', s') \\ 
  &= R(b, s, a) + \discount \underset{s'}{\sum}P(s' | b, s, a)\pivfn(b', s')
\end{align*}

Where $\actm = \plcypi(\stm, b)$, $R(b, s, a) = \underset{\parm_i}{\sum}R(\parm_i, s , a)b(\parm_i)$ is the mean reward function, $P(s' | b, s, a) = \underset{\parm_i}{\sum}\traFm(\parm_i, s, a, s')b(\parm_i)$ is the mean transition function. The second line follows from the fact that belief update is deterministic, \ie, $P(b' | b, \actm, \obsh) = 1$. The optimal Bayesian value function is
\begin{equation}
\begin{split}
\optvfn(b, s) &= \underset{a}{\max}\bigg{\{} R(b, s, a) + \discount \underset{s'}{\sum}P(s' | b, s, a)\optvfn(b', s') \bigg{\}}
\end{split}
\label{equ:optBayesV} 
\end{equation}

$a = \optplcypi(b, s)$ is the optimal action that maximizes the right hand size. Like the optimal policy in the original \hpmdp problem, the optimal Bayesian policy chooses actions not only based on how they affect the next state but also based on how they affect the next belief. 

However, the optimal Bayesian policy is computationally intractable. Instead of exploring by updating the belief each step, our algorithm explores by explicitly modify the reward function. In other words, each state action pair will have a reward bonus based on how much information it can reveal. The reward bonus used by our algorithm is motivated by the observation that the belief gets updated whenever some information about the hidden parameter has been revealed, thus we use the \Lone divergence between two beliefs to measure the amount of information gain. The reward bonus is defined formally as follows:

\begin{defn}
\label{def:rb}
When the belief is updated from $b_i$ to $b_j$, we measure the information gain by the \Lone divergence between $b_i$ and $b_j$, \ie, $\|b_j - b_i\|_1 = \underset{\mpr_i \in \Theta}{\sum}\big{|}b_{j}(\mpr_i)-b_i(\mpr_i)\big{|}$. Based on it, the reward bonus for \tuple{s, a} is defined as the expected $\Lone$ divergence between current belief $b$ and next belief $b_{s'}$:
\begin{align*}
\rebonus(b, s, a) = \disRB \underset{s'}{\ev}\big{[}\|b_{s'} - b\|_1\big{]} = \disRB \underset{s'}{\sum}P(s' | b, s, a)\|b_{s'} - b\|_1
\end{align*}
where \disRB is the constant tuning factor, $b_{s'}$ is the updated belief after observing $s'$.
\end{defn}

At each time step, our algorithm solves an internal reward MDP, $M = \tuple{\stspim, \actspim, \traFim, \rewFim, \discount, b}$. It chooses action greedily with respect to the following value function
\begin{equation}
\label{eq:Value}
\begin{split}
\optvfni(b, s) = \underset{a}{\max} \bigg{\{} \rewFim(b, s, a) + \discount \underset{s'}{\sum}P(s' | b, s, a)\optvfni(b, s') \bigg{\}}
\end{split}
\end{equation}

Where $\rewFim(b, s, a) = \rewFm(b, s, a) + \rebonus(b, s, a)$, in which $\rebonus(b, s, a)$ is the reward bonus term and it is defined in~\defref{def:rb}. Other parts are identical to~\equref{equ:optBayesV} except that belief $b$ is not updated in this equation. We can solve it using the standard Value Iteration algorithms, which have time complexity of $\bigO(|A||X|^2|\parsph|^2)$. In this work, we are more interested in problems with large state space, thus we are using \uct~\cite{kocsis2006bandit}, an \online \mdp solver, to achieve \online performance. Details of our algorithm is described in~\algmref{alg:bebd}, in which $\qfni^*(b_t,s_t,a) = \rewFim(b_t, s_t, a) + \discount \underset{s_{t+1}}{\sum}P(s_{t+1} | b_t, s_t, a)\optvfni(b_t, s_{t+1})$. 
\vspace{-0.8em}
\begin{algorithm}
\caption{}
\begin{algorithmic}[1]
\STATE $t \leftarrow 0, b_t \leftarrow b_0, s_t \leftarrow s_0$ \COMMENT{initialize the values}
\STATE $T \leftarrow $ Maximum Steps
\WHILE{not end and $t < T$}
\STATE $a_t \leftarrow \argmax_a \qfni^*(b_t,s_t,a)$ \COMMENT{greedily choose an action}
\STATE $(r_t,s_{t+1}) \leftarrow ExecuteAction(a_t)$ 
\STATE $b_{t+1} \leftarrow UpdateBelief(b_t,s_{t+1})$ \COMMENT{update the belief}
\STATE $t = t+1$
\ENDWHILE
\end{algorithmic}
\label{alg:bebd}
\end{algorithm}
\vspace{-0.8em}

\section{Analysis}
\label{sec:alg}
Although our algorithm is a greedy algorithm, it actually performs sub-optimally only in a polynomial number of time steps. In this section, we present some theoretical results to bound the sample complexity of our algorithm. Unless stated otherwise, the proof of the lemmas in this section are deferred to the appendix. For a clean analysis, we assume the reward function is bounded in $[0, 1]$.

\subsection{Sample Complexity}
The sample complexity measures the number of samples needed for an algorithm to perform optimally. We start with a definition of sample complexity on a state action pair \tuple{s, a}.
\begin{defn}
\label{def:sampComplexity}
Given the initial belief $b_0$, target accuracy $\sampAccu$, reward bonus tuning factor $\disRB$, we define the sample complexity function of \tuple{s, a} as: $\sampCplx(s, a) = f(b_0, s, a, \sampAccu, \disRB)$, such that if \tuple{s, a} has been visited more than $\sampCplx(s, a)$ times, starting from belief $b_0$, the corresponding reward bonus of visiting \tuple{s, a} at the new belief $b'$ is less than \sampAccu, \ie, $\rebonus(b', s, a) < \sampAccu$. We declare \tuple{s, a} as known if it has been sampled more than $\sampCplx(s, a)$ times, and cease to update the belief for sampling known state action pairs.
\end{defn} 

The following is an assumption for our theorem to hold true in general. The assumption essentially says that the earlier you try a state-action pair, the more information you can gain from it. We give a concrete example to illustrate our assumption in~\lemref{lem:dirichelet}.
\begin{assum}
The reward bonus monotonically decreases for all state action pairs \tuple{s, a} and timesteps $t$, \ie, $\rebonus(b_i, s, a) \geq \rebonus(b_{i+1}, s, a)$. 
\label{assum:monotonicityofrb}
\end{assum}

Now, we present our central theoretical result, which bounds the sample complexity of our algorithm with respect to the optimal Bayesian policy.

\begin{them_1}
Let the sample complexity of \tuple{s,a} be $\sampCplx(s,a) = f(b_0, s, a, \sampAccu, \disRB)$, where $\beta = O\big{(}\frac{|S|^2|A|}{(1 - \discount)^2}\big{)}$, $\sampAccu = \epsilon(1 - \discount)$. Let $\optply_t$ denote the policy followed by the algorithm at time $t$, and let $s_t$, $b_t$ be the corresponding state and belief. Then with probability at least $1-\delta$, $ V^{\optply_t}(b_t, s_t) \geq V^*(b_t,s_t) - 4 \epsilon $, \ie, the algorithm is $4 \epsilon$-close to the optimal Bayesian policy, for all but
\begin{center}
$m = O\bigg{(} \frac{\sum_{s,a}\sampCplx(s, a)}{\epsilon (1 - \discount)^3} \ln \frac{1}{\sigma} \ln \frac{1}{\epsilon(1 - \discount)} \bigg{)}$
\end{center}
time steps.
\label{them:theorem1}
\end{them_1}

In other words, our algorithm acts sub-optimally for only a polynomial number of time steps. 

Although our algorithm was primary designed for Discrete prior, \thmref{them:theorem1} can be applied to many prior distributions. We apply it to two simple special classes, which we can provide concrete sample complexity bound. First, we show that in the case of independent Dirichlet prior, the reward bonus monotonically decreases and the sample complexity of a \tuple{s, a} pair can be bounded by a polynomial function. This case also satisfies~\asmref{assum:monotonicityofrb}.

\begin{lemma}[Independent Dirichlet Prior]
\label{lem:dirichelet}
Let $n(s,a)$ be the number of times \tuple{s, a} has been visited. For a known reward function and an independent Dirichlet prior over the transition dynamics for each \tuple{s, a} pair, $\rebonus(b, s, a)$ monotonically decreases at the rate of $O\big{(}1/n(s, a)\big{)}$, and the sample complexity function $f(b_0, s, a, \sampAccu, \disRB) = O \big{(} \frac{|S|^2|A|}{\sampAccu (1 - \discount)^2} \big{)}$.
\end{lemma}

The strength of our algorithm lies in its ability to handle Discrete prior. We use a very simple example (Discrete prior over unknown deterministic MDPs) to show this advantage, and we state it in the following lemma. The intuition behind this lemma is quite simple, after sampling a state action pair, the agent will know its effect without noise.
\begin{lemma}[Discrete prior over Deterministic MDPs]
\label{lem:accurateSense}
Let $b_0$ be a Discrete prior over deterministic MDPs, the sample complexity function $f(b_0, s, a, \sampAccu, \disRB) \leq 1$.
\end{lemma}

\subsection{Proof of~\thmref{them:theorem1}}
The key intuition for us to prove our algorithm quickly achieves near-optimality is that at each time step our algorithm is $\epsilon$-optimistic with respect to the Bayesian policy, and the value of optimism decays to zero given enough samples.

The proof for~\thmref{them:theorem1} follows the standard arguments from previous \pacmdp results. We first show that $V^{\optply_t}(b_t, s_t)$ is close to the value of acting according to the optimal Bayesian policy, assuming the probability of escaping the known state-action set is small. Then we use the Hoeffding bound to show that this ``escaping probability'' can be large only for a polynomial number of time steps.

We begin our proof with the following lemmas. Our first lemma essentially says that if we solve the internal reward MDP using the current mean of belief state with an additional exploration bonus in~\defref{def:rb}, this will lead to a value function which is $\epsilon$-\emph{optimistic} to the Bayesian policy. 

\begin{lemma}[Optimistic]
\label{lem:upBound}
Let $\optvfni(b, s)$ be the value function in our algorithm, $\optvfn(b, s)$ be the value function in Bayesian policy. if $\beta = O\big{(}\frac{|S|^2|A|}{(1 - \disvfn)^2}\big{)}$, then $\forall s$, $\optvfni(b, s) \geq \optvfn(b, s) - \epsilon$.
\end{lemma}

The following definition is a generalization of the ``known state-action MDP'' in~\cref{strehl2012incremental} to Bayesian settings. It is an \mdp whose dynamics (transition function and reward function) are equal to the mean \mdp for \tuple{s, a} pairs in $K$ (known set). For other \tuple{s, a} pairs, the value of taking those \tuple{s, a} pairs in $M_K$ is equal to the current $\qfni$ value estimate.

\begin{defn}
\label{def:knownstateMDP}
Given current belief is $b$, a set of $\qfni$ value estimate for each \tuple{s, a} pair, \ie, $\qfni(b, \stim, \actim) = \rewFim(b, s, a) + \discount \underset{s'}{\sum}\traFim(b, s, a, s')\optvfni(b, s')$,  and a set $K$ of known \tuple{s, a} pairs, \ie, $\rebonus(b, s, a) < \epsilon(1 - \discount)$. We define the \knownMK, $M_K = \tuple{\stspim \cup \{s_0\}, \actspim, \traFim_K, \rewFm_K, \discount}$, as follows. $s_0$ is an additional state, under all actions from $s_0$ the agent returned to $s_0$ with probability $1$ and received reward $0$. For all $\tuple{s, a} \in K$, $\rewFm_K(s, a) = \rewFm(b, s, a)$ and $\traFm_K(s, a, s') = \traFm(b, s, a, s')$. For all $\tuple{s, a} \notin K$, $\rewFm_K(s, a) = \qfni(b, s, a)$ and $\traFm_K(s, a, s_0) = 1$.
\end{defn}

Our final lemma shows that the internal reward \mdp and the known state-action \mdp have low error in the set of known \tuple{s, a} pairs. 

\begin{lemma}[Accuracy]
\label{lem:accuracy}
Fix the history to the time step $t$, let $b_t$ be the belief, $s_t$ be the state, $K_t$ be the set of known \tuple{s, a} pairs, $M_{K_t}$ be the \knownMK, $\plcypi_t$ be the greedy policy with respect to current belief $b_t$, \ie, $\plcypi_t = \argmax_a \qfni(b_t, s, a)$. Then $\vfni^{\plcypi_t}(b_t, s_t) - \vfn_{M_{K_t}}^{\plcypi_t}(b_t, s_t) \leq \epsilon$.
\end{lemma}

Now, we are ready to prove~\thmref{them:theorem1}.

\begin{proof}[Proof of~\thmref{them:theorem1}]
\label{prf:theorem1}
Let $b_t, s_t, K_t, M_{K_t}, \plcypi_t$ be as described in~\lemref{lem:accuracy}. Let $T = 1 / (1 - \discount) \ln(1 / \epsilon(1 - \discount))$, then $| \vfn^{\plcypi_t}_{M_{K_t}}(b_t, s_t, T) - \vfn^{\plcypi_t}_{M_{K_t}}(b_t, s_t) | \leq \epsilon$ (see Lemma 2 of~\cite{kearns2002near}). Let \escapeE denote the event that, a \tuple{s, a} pair not in $K_t$ is generated when executing $\optply_t$ starting from $s_t$ for $T$ time steps. We have 
\begin{align*}
& \vfn^{\optply_t}(b_t, s_t, T) \geq \vfn^{\plcypi_t}_{M_{K_t}}(b_t, s_t, T) - P(\escapeE) / (1 - \discount)^2 \\ 
  & \geq \vfn^{\plcypi_t}_{M_{K_t}}(b_t, s_t) - \epsilon - P(\escapeE) / (1 - \discount)^2 \\ 
  & \geq \vfni^{\plcypi_t}(b_t, s_t) - 2 \epsilon - P(\escapeE) / (1 - \discount)^2 \\
  & \geq \optvfn(b_t, s_t) - 3 \epsilon - P(\escapeE) / (1 - \discount)^2
\end{align*}
The first inequality follows from the fact that $\optply_t$ equals to $\plcypi_t$ unless \escapeE occurs, and $\qfni(b, s, a)$ can be bounded by $1 / (1 - \discount)^2$ since we can limit the reward bonus to $1 / (1 - \discount)$ and still maintain optimism. The second inequality follows from the definition of $T$ above, the third inequality follows from~\lemref{lem:accuracy}, the last inequality follows from~\lemref{lem:upBound} and the fact that $\plcypi_t$ is precisely the optimal policy for the internal reward \mdp at time $t$. Now, suppose $P(\escapeE) < \epsilon (1 - \discount)^2$, we have $\vfn^{\optply_t}(b_t, s_t) \geq \vfn^{\optply_t}(b_t, s_t, T) \geq \optvfn(b_t, s_t) - 4 \epsilon$. Otherwise, if $P(\escapeE) \geq \epsilon (1 - \discount)^2$, by Hoeffding inequality, this will happen no more than $\bigOhead \bigg{(} \underset{s, a}{\sum}\sampCplx(s, a) T / \epsilon (1 - \discount)^2) \bigg{)}$ time steps with probability $1 - \delta$, where \bigOhead(.) notation suppresses logarithmic factors.
\end{proof}

\section{Experiments}
\label{sec:experiments}
To evaluate our algorithm experimentally, we compare it with several state of the art algorithms in \pomdp literature. \pomcp~\cite{SilVen10} and \despot~\cite{SomYe13} are two successful \online \pomdp planners which can scale to very large POMDPs. \qmdp~\cite{ThrBur05} is a myopic \offline solver being widely used for its efficiency. \sarsop~\cite{KurHsu08} is a state of the art \offline \pomdp solver which helps to calibrate the best performance achievable for POMDPs of moderate size. \mmdp is a common myopic approximation of Bayesian planning, which does not do exploration. For \sarsop, \pomcp and \despot, we used the software provided by the authors, with a slight modification on \pomcp to make it strictly follow 1-second time limit for planning. For our algorithm and \mmdp, a \mdp needs to be solved each step. We use an \online \mdp solver \uct~\cite{kocsis2006bandit} with similar parameter settings used in \pomcp. The reward bonus scalar \disRB used by our algorithm is typically much smaller than the one required by~\thmref{them:theorem1}, which is a common trend for internal reward algorithms. We tuned \disRB \offline before using it for planning.

\begin{table*}[ht]
\centering
\caption{Performance comparison on \rocksample and \battleship}
\vspace{-0.7em}
\begin{tabular}{ c c c c c c c }
\hline
\hline
& $RS(7,8)$ & $RS(11,11)$ & $RS(15,15)$ & $RS(20,20)$ & $BS(10,5)$ & $BS(15,7)$ \\ 
\hline
States $|\stsph|$ & $12,544$ & $247,808$ & $7,372,800$ & $10^{8}$ & $10^{13}$ & $10^{20}$ \\ 
Actions $|\actsph|$ & $13$ & $16$ & $20$ & $25$ & $100$ & $225$ \\ 
Observations $|\obssph|$ & $3$ & $3$ & $3$ & $3$ & $2$ & $2$ \\ 
\hline
\qmdp & $17.55 \pm 0.44$ & $16.10 \pm 0.40$ & $-$ & $-$ & $-$ & $-$ \\
\hline
\sarsop & $\mathbf{21.47 \pm 0.04}$ & $\mathbf{21.56 \pm 0.11}$ & $-$ & $-$ & $-$ & $-$ \\ 
\hline
\pomcp & $19.95 \pm 0.23$ & $20.11 \pm 0.23$ & $15.51 \pm 0.23$ & $12.11 \pm 0.23$ & $57.40 \pm 0.19$ & $119.33 \pm 0.58$ \\ 
\hline
\despot & $20.80 \pm 0.22$ & $21.12 \pm 0.21$ & $\mathbf{18.59 \pm 0.43}$ & $12.56 \pm 0.38$ & $56.34 \pm 0.22$ & $117.39 \pm 0.88$  \\ 
\hline
\mmdp & $15.11 \pm 0.17$ & $11.64 \pm 0.17$ & $7.53 \pm 0.16$ & $6.96 \pm 0.11$ & $57.46 \pm 0.17$ & $122.60 \pm 0.59$ \\ 
\hline
\hpmdp & $21.03 \pm 0.21$ & $\mathbf{21.52 \pm 0.20}$ & $\mathbf{18.63 \pm 0.20}$ & $\mathbf{16.81 \pm 0.20}$ & $\mathbf{58.16 \pm 0.17}$ & $\mathbf{127.12 \pm 0.61}$ \\ 
\hline
\end{tabular}
\vspace{-2.5em}
\label{tab:pomdpProb}
\end{table*}


We first apply the algorithms on two benchmarks problems in \pomdp literature, in which we demonstrate the scaling up ability of our algorithm on larger POMDPs. In $\rocksample(n,k)$~\cite{SmiSim05}, a robot moving in an $n \times n$ grid which contains $k$ rocks, each of which may be good or bad with probability $0.5$ initially. At each step, the robot can move to an adjacent cell, or sense a rock. The robot can only sample the rock when it is in the grid which contains a rock. Sample a rock gives a reward of $+10$ if the rock is good and $-10$ otherwise. Move or Sample will not produce observation. Sensing produces an observation in set $\obssph = \{Good, Bad\}$ with accuracy decreasing exponentially as the robot's distance to the rock increases. The robot reaches the terminal state when it passes the east edge of the map. The discount factor is $0.95$. The hidden parameter is the property of rock and it remains constant, thus this problem can be modeled as \hpmdp.

In $\battleship(n, k)$~\cite{SilVen10}, $k$ ships are placed at random into a $n \times n$ grid, subject to the constraint that no ship may be placed adjacent or diagonally adjacent to another ship. Each ship has a different size of $(k+1) \times 1, k \times 1, ... , 2 \times 1$. The goal is to find and sink all ships. Initially, the agent does not know the configuration of the ships. Each step, the agent can fire upon one cell of the grid, and receives observation $1$ if a ship was hit, otherwise it will receive observation $0$. There is a $-1$ reward per step, and a terminal reward of $n \times n$ for hitting every cell of every ship. It is illegal to fire twice on the same cell. The discount factor is $1$. The hidden parameter is the configuration of ships, which remains constant, thus this problem can also be modeled as \hpmdp. 

The results for \rocksample and \battleship are shown in~\tabref{tab:pomdpProb}. All algorithms, except for \qmdp and \sarsop (\offline algorithms), run in real time with $1$ second per step. The result for \sarsop was replicated from~\cite{ong2010planning}, other results are from our own test and were averaged over 1000 runs. ``$-$'' means the problem size is too large for the algorithm. $RS$ is short for \rocksample, $BS$ is short for \battleship. As we can see from~\tabref{tab:pomdpProb}, our algorithm achieves similar performance with the state of the art \offline solvers when the problem size is small. However, when the size of problem increases, \offline solvers start to fail and our algorithm outperforms other \online algorithms.

Finally, we show a robot arm grasping task, which is originated from Amazon Picking Challenge. A \vrep~\cref{rohmer2013v} simulation view is shown in~\figref{fig:grasp:simu}. The goal of the robot arm is to grasp the cup out of the shelf quickly and robustly. The robot knows its configuration exactly and its movement is deterministic. However, due to sensor limitations, the initial position of the cup is uncertain. The gripper has a tactile sensor inside each finger, which gives positive readings when the inner part of the finger gets in touch with the cup. The robot needs to move around to localize the cup, and grasp it as soon as possible. Usually, this can be modeled as a \pomdp problem. However, if we model it as a \hpmdp, our algorithm can achieve much better performance compared with solving it as a \pomdp. Now, we introduce our planning model for this task. We restrict the movement of the gripper in a $2D$ plane, as shown in~\figref{fig:grasp:model}. We divide the $2D$ plane into $3$ regions relative to the gripper. If the cup is in region $0$, the gripper can get in touch with the cup by moving along x-axis or y-axis. If the cup is in region $1$, the gripper can get in touch with the cup by moving along the y-axis. If the cup is in region $2$, the gripper can not sense the cup by moving in a single direction. The gripper can move along the $x$ or $y$ axis with step size of $0.01$. The reward for each movement is $-1$ in region $0$, $-2$ in region $1$ and $-5$ in region $2$. The gripper can close or open its fingers, with reward of $-10$. Picking the cup gives a reward of $100$ if the pick is successful and $-100$ otherwise.

We compare our algorithm with \pomcp, \despot, and \mmdp, since \qmdp and \sarsop do not support continuous state space. All algorithm are tested via model evaluation and \vrep simulation. Model evaluation means we use the planning model to examine the policy. \vrep simulation means we compute the best action using the planning model, then execute it in V-REP simulation. The next state and observation are obtained from \vrep. The results for model evaluation and \vrep simulation are reported in~\tabref{tab:grasp}. The time used for \online planning of all algorithms is $1$ second per step. We run $1000$ trials for model evaluation and $100$ trials for \vrep simulation. As we can see, our algorithm achieves higher return and success rate in both settings compared with other algorithms.


\begin{figure}
\centering
\begin{tabular}{c}
\subfloat[Simulation view \label{fig:grasp:simu}]{\includegraphics[width=2.5in]{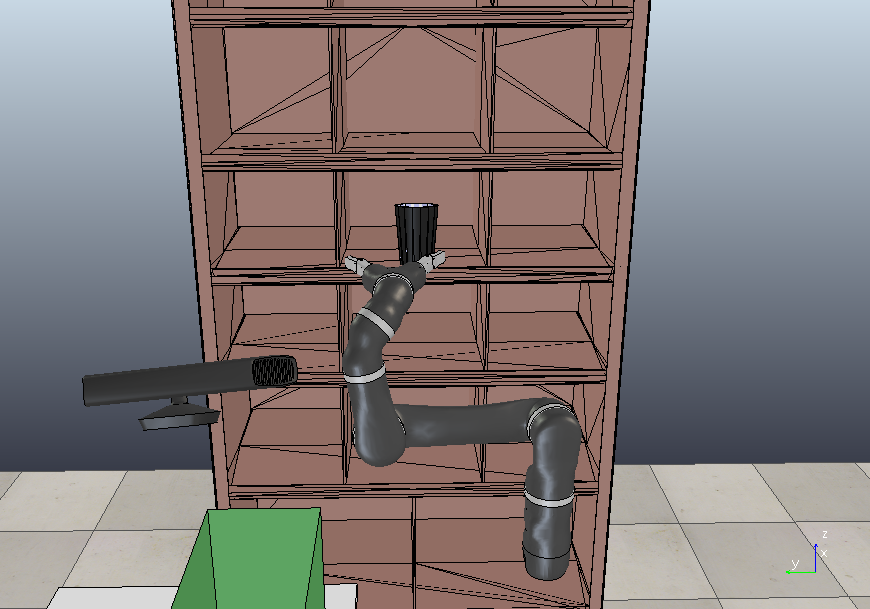}} \\
\subfloat[Simplified model \label{fig:grasp:model}]{\includegraphics[width=2.5in]{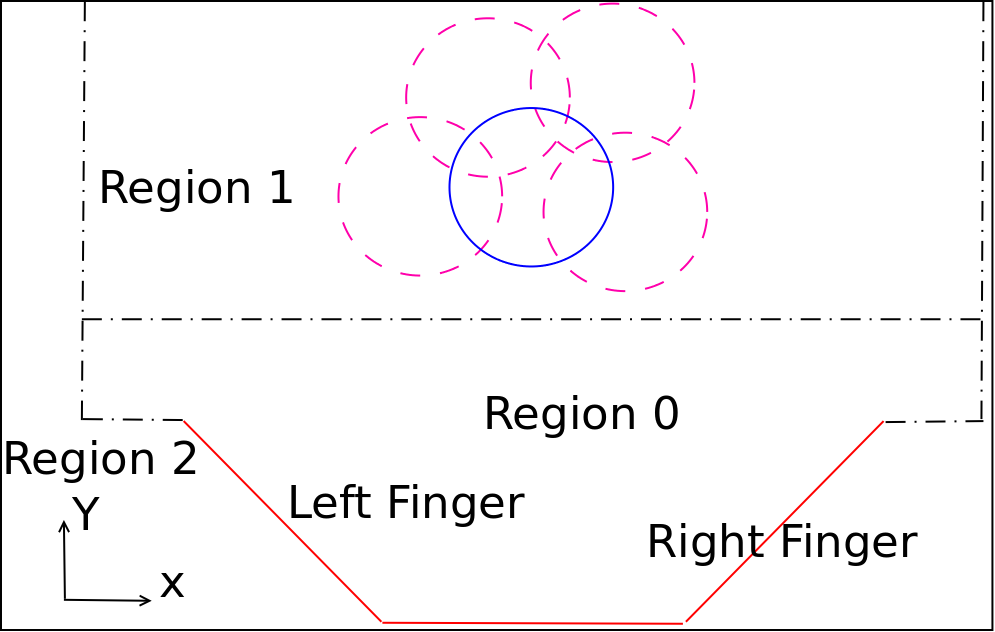}}
\end{tabular}
\caption{Robot arm grasping task}
\label{fig:grasp}
\end{figure}

\begin{table}
\centering
\caption{Performance comparison on grasping cup}
\vspace{-0.7em}
\begin{tabular}{@{\hspace{0em}} c @{\hspace{0em}} | @{\hspace{0.1em}} c @{\hspace{0.1em}} | @{\hspace{0.1em}} c @{\hspace{0.1em}}}
\hline
\hline
& Model Evaluation & V-REP simulation \\ 
\hline
States $|\stsph|$ & Continuous & Continuous \\
\hline
Actions $|\actsph|$ & 7 & 7 \\ 
\hline 
Observations $|\obssph|$ & 6 & 6 \\ 
\hline 
 & \begin{tabular}{c c}
Return & Success Rate 
\end{tabular} & \begin{tabular}{c c}
Return & Success Rate \\
\end{tabular} \\ 
\hline
\pomcp & \begin{tabular} {c c}
$14.17 \pm 0.36$ & $100\%$
\end{tabular} & \begin{tabular} {c c}
$12.91 \pm 2.10$ & $94\%$
\end{tabular} \\ 
\hline
\despot & \begin{tabular} {c c}
$11.35 \pm 0.46$ & $99.8\%$
\end{tabular} & \begin{tabular} {c c}
$9.63 \pm 2.52$ & $92\%$
\end{tabular} \\ 
\hline
\mmdp & \begin{tabular} {c c}
$14.16 \pm 0.46$ & $98.6\%$
\end{tabular} & \begin{tabular} {c c}
$11.28 \pm 1.81$ & $94\%$
\end{tabular} \\ 
\hline
\hpmdp & \begin{tabular} {c c}
$\mathbf{15.56 \pm 0.31}$ & $\mathbf{100\%}$
\end{tabular} & \begin{tabular} {c c}
$\mathbf{15.46 \pm 1.20}$ & $\mathbf{98\%}$
\end{tabular} \\ 
\hline
\end{tabular}
\vspace{-3.5em}
\label{tab:grasp}
\end{table}

\section{Conclusion}
\label{sec:conc}
We have introduced \hpmdp, a subclass of POMDP with hidden variables that are
either static or change deterministically.  A \hpmdp is equivalent to a set of
MDPs indexed by a hidden parameter.  By exploiting this equivalence, we have
developed a simple online algorithm for POMDP-lite through model-based
Bayesian reinforcement learning.  Preliminary experiments suggest that the
algorithm outperforms state-of-the-art general-purpose \pomdp solvers on very
large POMDP-lite models, makes it a promising tool for large-scale robot planning
under uncertainty.

Currently, we are implementing and experimenting with the algorithm on a
Kinova Mico robot for object manipulation. It is interesting and important to
investigate extensions that handle large observation and action spaces.

\addtolength{\textheight}{-0cm}   








\bibliography{icra16}

\begin{thebibliography}{10}
\providecommand{\url}[1]{#1}
\csname url@rmstyle\endcsname
\providecommand{\newblock}{\relax}
\providecommand{\bibinfo}[2]{#2}
\providecommand\BIBentrySTDinterwordspacing{\spaceskip=0pt\relax}
\providecommand\BIBentryALTinterwordstretchfactor{4}
\providecommand\BIBentryALTinterwordspacing{\spaceskip=\fontdimen2\font plus
\BIBentryALTinterwordstretchfactor\fontdimen3\font minus
  \fontdimen4\font\relax}
\providecommand\BIBforeignlanguage[2]{{%
\expandafter\ifx\csname l@#1\endcsname\relax
\typeout{** WARNING: IEEEtran.bst: No hyphenation pattern has been}%
\typeout{** loaded for the language `#1'. Using the pattern for}%
\typeout{** the default language instead.}%
\else
\language=\csname l@#1\endcsname
\fi
#2}}

\bibitem{asmuth2009bayesian}
J.~Asmuth, L.~Li, M.~L. Littman, A.~Nouri, and D.~Wingate, ``A bayesian
  sampling approach to exploration in reinforcement learning,'' in
  \emph{Proceedings of the Twenty-Fifth Conference on Uncertainty in Artificial
  Intelligence}, 2009.

\bibitem{BaiCai15}
H.~Bai, S.~Cai, D.~Hsu, and W.~Lee, ``Intention-aware online {POMDP} planning
  for autonomous driving in a crowd,'' in \emph{Proc. IEEE Int. Conf. on
  Robotics \& Automation}, 2015.

\bibitem{BaiHsu13a}
H.~Bai, D.~Hsu, and W.~Lee, ``Planning how to learn,'' in \emph{Proc. IEEE Int.
  Conf. on Robotics \& Automation}, 2013.

\bibitem{BanWon12}
T.~Bandyopadhyay, K.~Won, E.~Frazzoli, D.~Hsu, W.~Lee, and D.~Rus,
  ``Intention-aware motion planning,'' in \emph{Algorithmic Foundations of
  Robotics X---Proc. Int. Workshop on the Algorithmic Foundations of Robotics
  (WAFR)}, 2012.

\bibitem{DosKon13}
F.~Doshi-Velez and G.~Konidaris, ``Hidden parameter {Markov} decision
  processes: A semiparametric regression approach for discovering latent task
  parametrizations,'' arXiv preprint arXiv:1308.3513, 2013.

\bibitem{FerNat07}
A.~Fern, S.~Natarajan, K.~Judah, and P.~Tadepalli, ``A decision-theoretic model
  of assistance,'' in \emph{Proc. AAAI Conf. on Artificial Intelligence}, 2007.

\bibitem{HsiKae07}
K.~Hsiao, L.~Kaelbling, and T.~Lozano-P{\'e}rez, ``Grasping {POMDPs},'' in
  \emph{Proc. IEEE Int. Conf. on Robotics \& Automation}, 2007.

\bibitem{KaeLit98}
L.~Kaelbling, M.~Littman, and A.~Cassandra, ``Planning and acting in partially
  observable stochastic domains,'' \emph{Artificial Intelligence}, vol. 101,
  no. 1--2, pp. 99--134, 1998.

\bibitem{kearns2002near}
M.~Kearns and S.~Singh, ``Near-optimal reinforcement learning in polynomial
  time,'' \emph{Machine Learning}, vol.~49, no. 2-3, pp. 209--232, 2002.

\bibitem{kocsis2006bandit}
L.~Kocsis and C.~Szepesv{\'a}ri, ``Bandit based monte-carlo planning,'' in
  \emph{Machine Learning: ECML 2006}.\hskip 1em plus 0.5em minus 0.4em\relax
  Springer, 2006, pp. 282--293.

\bibitem{kolter2009near}
J.~Z. Kolter and A.~Y. Ng, ``Near-bayesian exploration in polynomial time,'' in
  \emph{Proceedings of the 26th Annual International Conference on Machine
  Learning}, 2009.

\bibitem{KovPol15}
M.~Koval, N.~Pollard, and S.~Srinivasa, ``Pre- and post-contact policy
  decomposition for planar contact manipulation under uncertainty,'' \emph{Int.
  J. Robotics Research}, 2015.

\bibitem{KurHsu08}
H.~Kurniawati, D.~Hsu, and W.~Lee, ``{SARSOP}: Efficient point-based {POMDP}
  planning by approximating optimally reachable belief spaces,'' in \emph{Proc.
  Robotics: Science \& Systems}, 2008.

\bibitem{NikRam15}
S.~Nikolaidis, R.~Ramakrishnan, K.~Gu, and J.~Shah, ``Efficient model learning
  from joint-action demonstrations for human-robot collaborative tasks,'' in
  \emph{Proc. ACM/IEEE Int. Conf. on Human-Robot Interaction}, 2015.

\bibitem{ong2010planning}
S.~C. Ong, S.~W. Png, D.~Hsu, and W.~S. Lee, ``Planning under uncertainty for
  robotic tasks with mixed observability,'' \emph{The International Journal of
  Robotics Research}, vol.~29, no.~8, pp. 1053--1068, 2010.

\bibitem{papadimitriou1987complexity}
C.~H. Papadimitriou and J.~N. Tsitsiklis, ``The complexity of markov decision
  processes,'' \emph{Mathematics of operations research}, vol.~12, no.~3, pp.
  441--450, 1987.

\bibitem{poupart2006analytic}
P.~Poupart, N.~Vlassis, J.~Hoey, and K.~Regan, ``An analytic solution to
  discrete bayesian reinforcement learning,'' in \emph{Proceedings of the 23rd
  international conference on Machine learning}.\hskip 1em plus 0.5em minus
  0.4em\relax ACM, 2006.

\bibitem{rohmer2013v}
E.~Rohmer, S.~P. Singh, and M.~Freese, ``V-rep: A versatile and scalable robot
  simulation framework,'' in \emph{Intelligent Robots and Systems (IROS), 2013
  IEEE/RSJ International Conference on}.\hskip 1em plus 0.5em minus 0.4em\relax
  IEEE, 2013, pp. 1321--1326.

\bibitem{RosPin08}
S.~Ross, J.~Pineau, S.~Paquet, and B.~Chaib-Draa, ``Online planning algorithms
  for {POMDPs},'' \emph{J. Artificial Intelligence Research}, vol.~32, no.~1,
  pp. 663--704, 2008.

\bibitem{RoyThr99}
N.~Roy and S.~Thrun, ``Coastal navigation with mobile robots,'' in
  \emph{Advances in Neural Information Processing Systems}.\hskip 1em plus
  0.5em minus 0.4em\relax The MIT Press, 1999, vol.~12, pp. 1043--1049.

\bibitem{SeiKur15}
K.~Seiler, H.~Kurniawati, and S.~Singh, ``An online and approximate solver for
  pomdps with continuous action space,'' in \emph{Proc. IEEE Int. Conf. on
  Robotics \& Automation}, 2015.

\bibitem{SilVen10}
D.~Silver and J.~Veness, ``{Monte-Carlo} planning in large {POMDPs},'' in
  \emph{Advances in Neural Information Processing Systems}, 2010.

\bibitem{SmiSim05}
T.~Smith and R.~Simmons, ``Point-based {POMDP} algorithms: Improved analysis
  and implementation,'' in \emph{Proc. Conf. on Uncertainty in Artificial
  Intelligence}, 2005.

\bibitem{SomYe13}
A.~Somani, N.~Ye, D.~Hsu, and W.~Lee, ``{DESPOT}: Online {POMDP} planning with
  regularization,'' in \emph{Advances in Neural Information Processing
  Systems}, 2013.

\bibitem{Son71}
E.~Sondik, ``The optimal control of partially observable {Markov} processes,''
  Ph.D. dissertation, Stanford University, Stanford, California, USA, 1971.

\bibitem{sorg2012variance}
J.~Sorg, S.~P. Singh, and R.~L. Lewis, ``Variance-based rewards for approximate
  bayesian reinforcement learning,'' in \emph{{UAI}, Proceedings of the
  Twenty-Sixth Conference on Uncertainty in Artificial Intelligence}, 2010.

\bibitem{strehl2012incremental}
A.~L. Strehl, L.~Li, and M.~L. Littman, ``Incremental model-based learners with
  formal learning-time guarantees,'' in \emph{{UAI}, Proceedings of the 22nd
  Conference in Uncertainty in Artificial Intelligence}, 2006.

\bibitem{ThrBur05}
S.~Thrun, W.~Burgard, and D.~Fox, \emph{Probabilistic Robotics}.\hskip 1em plus
  0.5em minus 0.4em\relax The MIT Press, 2005.

\end{thebibliography}

\newpage
\section*{Techinical Proofs}
\subsection{Proof of~\lemref{lem:dirichelet}}
\begin{proof}
\label{prf:dirichlet} 
Let $\alpha_0$ denote $n(s,a)$, and let $\alpha_i$ denote $n(s,a,s_i)$. According the definition of the Dirichlet distribution, $P(s_j|b,s,a) = \frac{\alpha_j}{\alpha_0}$. The reward bonus term can be described as
\begin{align*}
&\frac{\rebonus(b, s, a)}{\disRB} = \underset{s_j \in S}{\sum}|P(s_j|b',s,a) - P(s_j|b,s,a)| \\ 
& = \underset{s_k}{\ev} \bigg{[} \underset{s_j \neq s_k}{\sum}\big{(}\frac{\alpha_j}{\alpha_0} - \frac{\alpha_j}{\alpha_0 + 1}\big{)} + \big{(}\frac{\alpha_k + 1}{\alpha_0 + 1} - \frac{\alpha_k}{\alpha_0}\big{)} \bigg{]} \\ 
&= P(s_k|b,s,a) \bigg{(} \underset{s_j \neq s_k}{\sum} \frac{\alpha_j}{\alpha_0 + \alpha_0^2} + \frac{\alpha_0 - \alpha_k}{\alpha_0 + \alpha_0^2} \bigg{)} \\
& \leq \underset{s_j}{\sum} \frac{\alpha_j}{\alpha_0 + \alpha_0^2} + \frac{\alpha_0}{\alpha_0 + \alpha_0^2} = \frac{2}{1 + \alpha_0}
\end{align*}
As for the sample complexity, if $\rebonus(b, s, a) < \sampAccu$, we have
\begin{align*}
&\rebonus(b, s, a) \leq \frac{2 \disRB}{1 + \alpha_0} \leq \sampAccu \Rightarrow \alpha_0 = O \big{(} \frac{\disRB}{\sampAccu} \big{)}= O \big{(} \frac{|S|^2|A|}{\sampAccu (1 - \discount)^2} \big{)}
\end{align*}
\end{proof}

\subsection{Proof of~\lemref{lem:upBound}}

We first introduce some notations which will be used in the proof. Denote a $T$ step history as $h_T = \{ b_0, s_0, a_0, b_1, s_1, a_1, ... , b_T, s_T \}$, where $b_t$ and $s_t$ is the belief and state at time step $t$. The following is a definition of \Lone divergence on reward function and transition function.

\begin{defn}
\label{def:L1DivergenceOnRT}
Denote \mRsp as the set of mean reward function, $\mTsp$ as the set of mean transition function. \ie, given belief $b$, $\mRsp(b) = \{R(b, s, a) | s \in S, a \in A\}$, $\mTsp(b) = \{T(b, s, a, s') | s, s' \in S, a \in A\}$. Suppose the belief changes from $b_i$ to $b_j$, the $\Lone$ divergence of $\mathbb{R}$ and $\mathbb{T}$ are denoted as:
\begin{align*}
& \mRsp(b_i, b_j) = \|\mRsp(b_j) - \mRsp(b_i)\|_1 = \underset{s, a}{\sum}\big{|}R(b_j, s, a)-R(b_i, s, a)\big{|} \\
&\mTsp(b_i, b_j) = \|\mTsp(b_j) - \mTsp(b_i)\|_1 \\ 
& = \underset{s, a}{\sum}\underset{s'}{\sum}\big{|}T(b_j, s, a, s') - T(b_{i}, s, a, s')\big{|}
\end{align*}
\end{defn} 

Based on~\defref{def:L1DivergenceOnRT}, we introduce the regret of a $T$ step history if the belief is not updated each step. 

\begin{defn}
\label{def:RegretOfPolicy}
Given a $T$ step history $h_T$, if the belief is not updated, the regret of the \ith action is defined as $\vE(b_0, i) = \mRsp (b_0,b_i) + \upbvi \mTsp(b_0,b_i)$. Define $\accRE(b_0, T) = \underset{i = 0}{\overset{T}{\sum}} \discount^i \vE(b_0, i, \plcypi)$ as the total regret of the history.
\end{defn}

The following definition measures the extra value from reward bonus term when we are using internal reward.

\begin{defn}
\label{def:BonusPolicy}
Given a policy $T$ step history $h_T$, the reward bonus for the \ith action is $\vB(b_0, i) = \rebonus(b_0, s_i, a_i)$. Define $\accRB(b_0, T) = \underset{i = 0}{\overset{T}{\sum}} \discount^i \vB(b_0, i)$ as the total extra value from reward bonus.
\end{defn}

In the next lemma, we are going to bound the regret using the extra value from reward bonus.

\begin{lemma}
\label{lem:rbUpperBound}
Given a $T$ step history $h_T$, let $\accRE(b_0, T)$ be the regret of not updating the belief, as defined in~\defref{def:RegretOfPolicy}. Let $\accRB(b_0, T)$ be the extra value from reward bonus term, as defined in~\defref{def:BonusPolicy}. If the constant tunning factor of reward bonus $\disRB = \bigO \big{(} \frac{|S|^2|A|T}{(1 - \discount)} \big{)}$, then $\accRB(b_0, T) \geq \accRE(b_0, T)$.
\end{lemma}

\begin{proof}
\label{prf:rbUpperBound}
We begin the proof by showing that the $\Lone$ divergence of the reward function is bounded by the reward bonus if \disRB was chosen properly. 
\begin{align*}
&\mRsp(b_0, b_i) = \underset{s ,a}{\sum}\big{|}R(b_i, s, a) - R(b_0, s, a)\big{|} \\ 
&= \underset{s, a}{\sum}\big{|}\underset{\parm_i \in \parsph}{\sum}R(\parm_i, s, a)(b_i(\parm_i) - b_0(\parm_i))\big{|} \\
& \leq \underset{s, a}{\sum}\underset{\parm_i \in \parsph}{\sum}|b_i(\parm_i) - b_0(\parm_i)| \\ 
&= |S||A|\|b_i - b_0\|_1 \leq |S||A| \overset{i - 1}{\underset{l = 0}{\sum}} \|b_{l + 1} - b_{l}\|_1\\ 
&= \frac{|S||A|}{\disRB}\overset{i - 1}{\underset{l = 0}{\sum}}\rebonus(b_l, s_l, a_l) \\
& \leq \frac{|S||A|}{\disRB}\overset{i - 1}{\underset{l = 0}{\sum}}\rebonus(b_{0}, s_l, a_l) \\ 
&= \frac{|S||A|}{\disRB} \overset{i - 1}{\underset{l = 0}{\sum}} \vB(b_0, l)
\end{align*}

The first inequality above follows from the fact that $R(\parm_i, s, a)$ is bounded in $[0, 1]$, and triangle inequality. The second inequality also follows from triangle inequality, \ie, $ \|b_{s_i} - b_0\|_1 \leq \|b_{s_i} - b_{s_{i - 1}}\|_1 + ... +  \|b_{s_1} - b_{0}\|_1 $. The third inequality follows from our monotonicity assumption on the reward bonus: $\rebonus(b_{l}, s_l, a_l) \leq \rebonus(b_0, s_l, a_l)$. 

Similarly, we can show that the \Lone divergence of the transition function is bounded by the reward bonus.

\begin{align*}
\mTsp(b_0,b_{s_i}) \leq \frac{|S|^2|A|}{\disRB}\overset{i - 1}{\underset{l = 0}{\sum}}\vB(b_{0}, l)
\end{align*}

Finally, we are going to show that the total regret $\accRE(b_0, T)$ of $h_T$ can be bounded by the extra value from reward bonus $\accRB(b_0, T)$:

\begin{align*}
&\accRE(b_0, T) = \overset{T}{\underset{i = 0}{\sum}} \discount^i \vE(b_0, i) \\
& = \overset{T}{\underset{i = 0}{\sum}} \discount^i \bigg{(}\mRsp(b_0,b_i) + \upbvi \mTsp(b_0,b_i)\bigg{)} \\ 
&\leq \bigO \bigg{(} \frac{|S|^2|A|}{\disRB (1 - \discount)} \bigg{)} \overset{T}{\underset{i = 0}{\sum}} \discount^i \bigg{(} \overset{i - 1}{\underset{l = 0}{\sum}} \vB(b_{0}, l) \bigg{)} \\
& \leq \bigO \bigg{(} \frac{|S|^2|A| T}{\disRB (1 - \discount)} \bigg{)} \overset{T}{\underset{i = 0}{\sum}} \discount^i \vB(b_{0}, i) = \accRB(b_0, i)
\end{align*}

The first inequality above follows from the fact that the \Lone divergence of reward function and transition is bounded by the reward bonus for some value of \disRB. The second inequality follows $\overset{T}{\underset{i = 0}{\sum}} \discount^i \bigg{(} \overset{i - 1}{\underset{l = 0}{\sum}} \vB(b_{0}, l) \bigg{)} \leq \overset{T - 1}{\underset{i = 0}{\sum}} \discount^i (T - i) \vB(b_{0}, i) \leq T \overset{T}{\underset{i = 0}{\sum}} \discount^i \vB(b_{0}, i) $.

\end{proof}

Now, we are ready to prove~\lemref{lem:upBound}. 

\begin{proof}[Proof of~\lemref{lem:upBound}]
\label{prf:upBound}
Let $T = 1 / (1 - \discount) \ln(1 / \epsilon(1 - \discount))$, then $| \optvfn(b_0, s_0, T) - \optvfn(b_0, s_0) | \leq \epsilon$ (see Lemma 2 of~\cite{kearns2002near}), where $b_0$ and $s_0$ is the initial belief and state. Consider some state $s$, let $b_t$ be the new belief formed by updating $b_0$ after $t \leq T$ steps, then

\begin{equation}
\label{eq:upBound}
\begin{split}
& \accRB(b_0, t-1) - \accRE(b_0, t - 1) + \\
& \optvfni_{T-t}(b_0,s) - \optvfn_{T-t}(b_t, s) \\ 
& = \accRB(b_0, t-1) - \accRE(b_0, t - 1) + \\
& \underset{a}{\max}\bigg{\{} R(b_0, s, a) + \rebonus(b_0, s, a) + \\ 
& \underset{s'}{\sum}P(s' | b_0, s, a)\optvfni_{T-t-1}(b, s') \bigg{\}} - \\ 
& \underset{a}{\max}\bigg{\{} R(b_t, s, a) + \underset{s'}{\sum}P(s' | b_t, s, a)\optvfni_{T-t-1}(b_{t+1}, s') \bigg{\}} \\ 
& \geq \accRB(b_0, t-1) - \accRE(b_0, t - 1) + \\
& \underset{a}{\min}\bigg{\{} \rebonus(b_0, s, a) - \big{(}R(b_t, s, a\big{)} - R(b_0, s, a)) + \\ 
& \underset{s'}{\sum}P(s' | b_0, s, a)\optvfni_{T-t-1}(b_0, s') - \\ 
& \underset{s'}{\sum}P(s' | b_t, s, a)\optvfni_{T-t-1}(b_{t+1}, s') \bigg{\}} \\ 
& \geq \accRB(b_0, t-1) - \accRE(b_0, t - 1) + \\ 
& \underset{a}{\min}\bigg{\{} \rebonus(b_0, s, a) - \big{(}R(b_t, s, a\big{)} - R(b_0, s, a)) - \\
& T \underset{s'}{\sum}|P(s | b_t, s, a) - P(s' | b_0, s, a)| + \\ 
& \underset{s'}{\sum}P(s' | b_0, s, a) \big{(} \optvfni_{T-t-1}(b, s') - \optvfni_{T-t-1}(b_{t+1}, s') \big{)} \bigg{\}} \\ 
& \geq \accRB(b_0, t) - \accRE(b_0, t) + \\
& \underset{a}{\min}\bigg{\{} \underset{s'}{\sum}P(s' | b_0, s, a) \big{(} \optvfni_{T-t-1}(b, s') - \optvfni_{T-t-1}(b_{t+1}, s') \bigg{\}} \\ 
& \geq \accRB(b_0, t) - \accRE(b_0, t) + \\ 
& \underset{s}{\min}\big{\{} \optvfni_{T-t-1}(b, s) - \optvfni_{T-t-1}(b_{t+1}, s) \big{\}} 
\end{split}
\end{equation}

The first inequality transformation in~\equref{eq:upBound} follows:
\begin{align*}
\underset{x}{\max}f(x) - \underset{x}{\max}g(x) \geq \underset{x}{\min} \big{(} f(x) - g(x) \big{)}
\end{align*}

The second inequality transformation in \equref{eq:upBound} follows:
\begin{equation*}
\begin{split}
&\underset{x}{\sum}p(x)f(x)-\underset{x}{\sum}q(x)g(x) \\
& \geq \underset{x}{\sum}p(x)(f(x)-g(x))-\underset{x}{\sum}|q(x)-p(x)|g(x)
\end{split}
\end{equation*}

The third inequality transformation in \equref{eq:upBound} follows the definition of \accRB and \accRE.

Since $s$ is arbitrary in~\equref{eq:upBound}, we have for any history $h_T$ and $t \leq T$,
\begin{equation}
\begin{split}
& \accRB(b_0, t-1) - \accRE(b_0, t - 1) + \\
& \optvfni_{T-t}(b_0,s) - \optvfn_{T-t}(b_t, s) \\ 
& \geq \accRB(b_0, t) - \accRE(b_0, t) + \\ 
& \underset{s}{\min}\big{\{} \optvfni_{T-t-1}(b, s) - \optvfni_{T-t-1}(b_{t+1}, s) \big{\}} 
\end{split}
\label{eq:upBound1}
\end{equation}

Where $\accRB(b_0, t) = 0, \accRE(b_0, t) = 0$ if $t \leq 0$.

Apply~\equref{eq:upBound1} repeatedly for $T$ steps, we get 
\begin{align*}
& \optvfni_{T}(b_0,s) - \optvfn_{T}(b_0, s) \\ 
& \geq \accRB(b_0, T) - \accRE(b_0, T) + \\ 
& \underset{s}{\min}\big{\{} \optvfni_{T-T}(b, s) - \optvfni_{T-T}(b_{T+1}, s) \big{\}} \geq 0
\end{align*}

The second inequality above follows~\lemref{lem:rbUpperBound} then we have $\optvfni(b_0,s) \geq \optvfni_{T}(b_0,s) \geq \optvfn_{T}(b_0, s) \geq \optvfn(b_0, s) - \epsilon$, which proves our lemma.

\end{proof}

\subsection{Proof of~\lemref{lem:accuracy}}
\begin{proof}
\label{prf:accuracy}
Consider the sequence of states, actions by following $\plcypi_t$ at state $s_t$ and belief $b_t$, $p_t = s_t,a_t,s_t^1,a_t^1, ... ,s_t^m,a_t^m$. Suppose \tuple{s_t^k, a_t^k} is the first state action pair not in $K$ is generated. We have $\vfni^{\plcypi}(b_t, s_t) = \overset{k - 1}{\underset{i = 0}{\sum}} \discount^i \rewFim(b_t, s_t^i, a_t^i) + \discount^k \qfni(b_t, s_t^k, a_t^k)$ and $\vfn_{M_{K_t}}^{\plcypi}(b_t, s_t) = \overset{k - 1}{\underset{i = 0}{\sum}} \discount^i \rewFm(b_t, s_t^i, a_t^i) + \discount^k \qfni(b_t, s_t^k, a_t^k)$. Then, $\vfni^{\plcypi}(b_t, s_t) - \vfn_{M_{K_t}}^{\plcypi}(b_t, s_t) \leq \overset{k - 1}{\underset{i = 0}{\sum}} \discount^i \epsilon(1 - \discount) \leq \epsilon$.
\end{proof}

\end{document}